\documentclass{article} %
\usepackage{hyperref}
\usepackage{url}

\newcommand{\bbE}{\mathbb{E}}
\newcommand{\bbR}{\mathbb{R}}

\newcommand{\cD}{\mathcal{D}}
\newcommand{\cP}{\mathcal{P}}

\newcommand{\lp}{\left(}
\newcommand{\rp}{\right)}

\usepackage[dvipsnames]{xcolor}

\usepackage{enumitem}
\usepackage[utf8]{inputenc}
\usepackage{booktabs} %
\usepackage{amsfonts} %
\usepackage{nicefrac} %
\usepackage{microtype} %
\usepackage{graphicx}
\usepackage%
{natbib}
\usepackage{gensymb}
\usepackage{color}
\usepackage{caption}
\usepackage{csquotes}
\usepackage{subcaption}
\usepackage{float}
\usepackage{amsthm}
\usepackage{mathrsfs}
\usepackage{amsmath}
\usepackage{amssymb}
\usepackage{mathtools}
\usepackage{wrapfig}
\usepackage{soul}
\usepackage[textwidth=1.2in,textsize=tiny]{todonotes}
\usepackage{diagbox}
\usepackage{algorithm}
\usepackage{algorithmic}
\usepackage{hhline}
\usepackage{multirow}
\usepackage{adjustbox}
\usepackage{textcomp}

\definecolor{codegreen}{rgb}{0,0.3,0.6}
\definecolor{codegray}{rgb}{0.5,0.5,0.5}
\definecolor{codepurple}{rgb}{0.58,0,0.82}
\definecolor{backcolour}{rgb}{0.95,0.95,0.92}
\definecolor{orange}{rgb}{1,0.5,0}

\usepackage{listings}
\lstdefinestyle{mystyle}{
    basicstyle=\tiny,
    commentstyle=\color{codegreen},
    keywordstyle=\color{magenta},
    numberstyle=\tiny\color{codegray},
    stringstyle=\color{codepurple},
    basicstyle=\fontsize{8.5}{9}\selectfont\ttfamily,
    breakatwhitespace=false,         
    breaklines=true,                 
    captionpos=b,                    
    keepspaces=true,                 
    numbers=none,                    
    numbersep=5pt,                  
    showspaces=false,                
    showstringspaces=false,
}
\usepackage{hhline}
\usepackage{multirow}

\usepackage[accepted]{icml2023}

\newcommand{\E}{\mathbb{E}}

\newcommand{\ba}[1]{\begin{align}#1\end{align}}

\newcommand{\cdotv}{\boldsymbol{\cdot}}

\usepackage{stackengine}

\makeatletter
\newcommand{\distas}[1]{\mathbin{\overset{#1}{\kern\z@\sim}}}%

\newcommand{\kl}{\text{KL}}

\newcommand{\beqs}{\vspace{0mm}\begin{eqnarray}}
\newcommand{\eeqs}{\vspace{0mm}\end{eqnarray}}
\newcommand{\barr}{\begin{array}}
\newcommand{\earr}{\end{array}}

\newcommand{\given}{\mid}
\newcommand{\divbar}{~ \| ~}

\newtheorem{theorem}{Theorem}

\newtheorem{lemma}{Lemma}

\icmltitlerunning{Thinning and Thickening for Generative Modeling}

\begin{document}

\twocolumn[
\icmltitle{Learning to Jump: Thinning and Thickening Latent Counts\\ for Generative Modeling}

\icmlsetsymbol{equal}{*}

\begin{icmlauthorlist}
\icmlauthor{Tianqi Chen}{xxx}
\icmlauthor{Mingyuan Zhou}{xxx}
\end{icmlauthorlist}

\icmlaffiliation{xxx}{McCombs School of Business, The University of Texas at Austin}

\icmlcorrespondingauthor{Tianqi Chen}{tqch@utexas.edu}
\icmlcorrespondingauthor{Mingyuan Zhou}{mingyuan.zhou@mccombs.utexas.edu}

\icmlkeywords{Machine Learning, ICML}

\vskip 0.3in
]
\printAffiliationsAndNotice{\icmlEqualContribution}

\begin{abstract}
Learning to denoise has emerged as a prominent paradigm to design state-of-the-art deep generative models for natural images. How to use it to model the distributions of both continuous real-valued data and categorical data has been well studied in recently proposed diffusion models. However, it is found in this paper to have limited ability in modeling some other types of data, such as count and non-negative continuous data, that are often highly sparse, skewed, heavy-tailed, and/or overdispersed. To this end, we propose learning to jump as a general recipe for generative modeling of various types of data. Using a forward count thinning process to construct learning objectives to train a deep neural network, it employs a reverse count thickening process to iteratively refine its generation through that network. We demonstrate when learning to jump is expected to perform comparably to learning to denoise, and when it is expected to perform better. For example, learning to jump is recommended when the training data is non-negative and exhibits strong sparsity, skewness, heavy-tailedness, and/or heterogeneity.

\end{abstract}

\section{Introduction}

Learning how to generate realistic artificial data from random noise is a foundational problem in statistics and machine learning. A common practice to address this problem is to employ deep generative models (DGMs), which include 
variational auto-encoders (VAEs) \citep{kingma2013auto,rezende2014stochastic}, %
autoregressive models \citep{van2016pixel,van2017neural,ramesh2021zero,ramesh2022hierarchical}, and generative adversarial networks (GANs) \citep{goodfellow2014generative} as representative examples. 
These DGMs usually generate a random data sample by forward propagating  a random noise through a decoder, empowered by deep neural networks, or producing the elements of a sequence in an autoregressive manner, via the use of a recurrent neural network \cite{hochreiter1997long,graves2008offline} or a Transformer \citep{vaswani2017attention,radford2018improving}.

Different from previous generative modeling frameworks, learning to denoise, which generates a sample through iterative refinement,   has recently emerged as a prominent paradigm in designing DGMs  \citep{sohl2015deep,scorematching,ddpm,song2021scorebased}.   
In this paradigm, one first corrupts the clean data with noise at various signal-to-noise ratios (SNRs), and then learns how to iteratively denoise the noisy data, using the same deep neural network that is made aware of the corresponding SNR \citep{kingma2021variational}.
In general, DGMs built under this learning-to-denoise framework are shown to convincingly outperform previous ones in training stability, mode coverage, and generation quality \citep{Dhariwal2021DiffusionMB,rombach2022high,saharia2022photorealistic}. 

Commonly formulated as either denoising diffusion probabilistic models (DDPMs) \citep{sohl2015deep,ddpm}  or score-based generative models \citep{JMLR:v6:hyvarinen05a,vincent2011connection,scorematching,improvedscore}, 
learning-to-denoise-based DGMs have been successfully used to model high-dimensional distributions of both continuous real-valued data \citep{%
Dhariwal2021DiffusionMB,ho2022cascaded,ramesh2022hierarchical,rombach2022high,saharia2022photorealistic} and categorical data \citep{hoogeboom2021argmax,austin2021structured,gu2022vector,hu2022global}. Despite being relatively new, they have already  been deployed into a diverse set of applications, including personalized image editing \citep{ruiz2022dreambooth}, audio synthesis \citep{chen2021wavegrad,kong2021diffwave,yang2023diffsound}, text generation \citep{li2022diffusionlm}, uncertainty quantification in classification and regression \citep{han2022card}, learning expressive policies in reinforcement learning \citep{wang2023diffusionrl}, and generation of chemical and biological compounds \citep{shi2021learning,luo2022antigenspecific,jing2022torsional}, to name a few. 
While learning to denoise so far has been successfully applied to both continuous real-valued and categorical data,  non-trivial modifications to this framework are in general  required on a case-by-case basis to accommodate every new type of data, such as  count and sparse non-negative data. 
To be more specific, let's consider the case of modeling sparse data, which is prevalent in numerous real-world applications. Examples include users' ratings of movies, consumers' purchases of products, term-frequency vectors in documents, next-generation RNA-sequencing data, and graph adjacency matrices, among others. Gaussian-based denoising diffusion models, however, 
are known to be inherently restrictive in modeling %
exact sparsity, as will also be illustrated in our experiments.

\begin{figure*}[t]
    \centering
    \includegraphics[width=.95\linewidth]{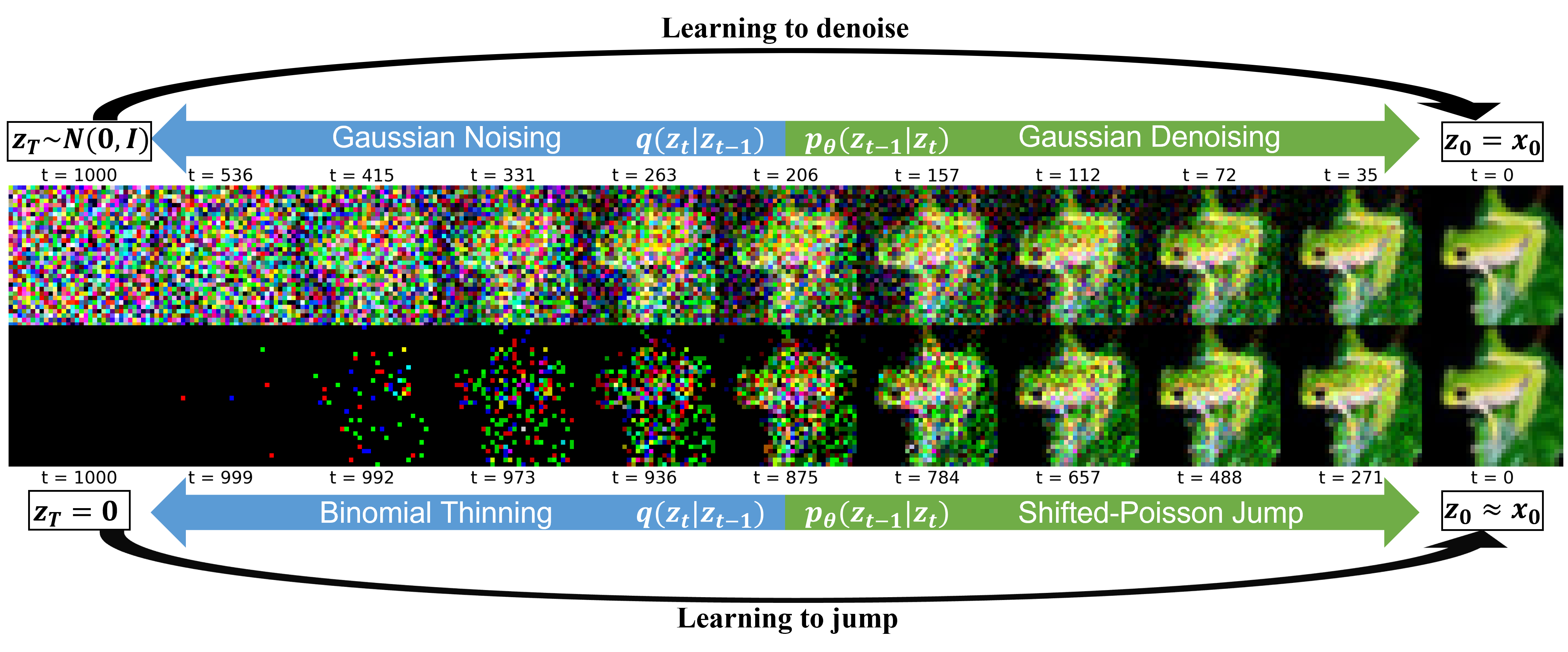}
    \caption{Illustrative comparison of learning to denoise, which adds noise for model training and denoise for data generation,  and learning to jump, which thins the counts  for model training and thickens the counts with Poisson jumps for data generation.}
    \label{fig:l2jvsl2d}
\end{figure*}

Distinct from learning to denoise, we propose learning to jump, as depicted in Figure~\ref{fig:l2jvsl2d}, that exploits the thinning and thickening of latent counts under the Poisson distribution to build a general  framework for deep generative modeling. We refer to the DGMs built under this new framework as JUMP models, %
which are suitable for any type of data that takes non-negative values, such as count, binary, %
and sparse non-negative continuous data. %
While existing DGMs  often start their generation from random noise and move to distinct states when the generation ends, the proposed thinning and thickening-based JUMP models start the data generation process from exact zeros and can stay at exact zeros when the generation ends.

JUMP models are implemented under a Bayesian framework equipped with a multi-stochastic-layer generative network, each layer of which shares the same deep neural network-based generator.  %
We first show any univariate non-negative observation can be recovered from a Poisson-distributed count according to the strong law of large numbers, and hence any univariate non-negative observation can be converted into a latent count with  a controllable accuracy to recover its original value.
We show in this latent count space, one can first thin the latent counts for model training and then iteratively thicken the latent counts with discrete non-negative jumps for data generation. 

More specifically, we define the decoder of a JUMP model via a Poisson distribution-based Markov chain, with the inference network defined via a forward Poisson thinning process. 
The Poisson Markov chain 
performs iterative refinement through a deep neural network. At each refinement step, the input consists of the cumulative count of all previous steps and the time embedding of the current step, while  the output is a shifted Poisson-distributed random count.
To train this deep neural network, 
we draw a Poisson-distributed count based on the observation and send it through a forward Poisson thinning process that gradually thins each count towards zero. Therefore, the number of non-zero locations during training (generation) is monotonically decreasing (increasing) as the number of diffusion (reverse-diffusion) steps increases.

The major contributions of the paper include: 1) Introducing learning to jump as a novel framework to build DGMs; 2) Revealing that learning to denoise has limited ability in modeling non-negative data that exhibit one or multiple features from the following list: sparsity, skewness, heavy-tailedness, and overdispersion; 3) Demonstrating the ability of learning-to-jump-based DGMs in modeling complex data.

\section{Learning to Jump}

The proposed jump diffusion probabilistic models, as depicted in Figure\,\ref{fig:l2jvsl2d}, can be roughly described as follows. First, we re-scale the data observations by a positive constant and then encode the re-scaled observations into Poisson-distributed latent counts. The applied encoding can be lossless and admits a simple approximate inverse. Second, we decrease the latent counts towards zeros through a series of binomial distribution-based thinning operations. Third, we reverse the thinning process via a count-thickening jump process, which increases the latent counts by a series of Poisson-distributed discrete jumps. We also note that if the data itself is already count-valued ($i.e.$, non-negative integers), then the initial operation of re-scaling and randomization via Poisson is not necessary. We defer the detailed discussion of that specific case into Appendix~\ref{sec:binomial}.
\subsection{Poisson-based Data Randomization}
Let us denote $\mathbb{N}_0:=\{0,1,2,\ldots\}$ as the set of non-negative integers and express the distribution of the observed  non-negative data~as $$x_0\sim \mathcal P_0\, .$$
 We first encode the observation $x_0$ into a count variable $z_0\in\mathbb{N}_0$ via the Poisson distribution, whose rate is defined by the re-scaled observation $\lambda x_0$ where $\lambda > 0$: %
$$
z_0\sim \mbox{Pois}(\lambda x_0)\, .
$$
The resulting latent count distribution can be viewed as a mixed Poisson distribution~\cite{karlis2005mixed} of which the mixing distribution is the data distribution and is identifiable\footnote{The term ``identifiable'' means being able to identify the mixing distribution of a mixture.}. 

For $x_0>0$, the standard-deviation-to-mean ratio monotonically decreases towards zero as $\lambda$ increases to infinity: 
$$
\frac{\mbox{std}(z_0)}{\E(z_0)} = \frac{1}{(\sqrt{\lambda  x_0})} \to 0\, .
$$
Moreover, when $\cP_0$ is a Dirac distribution $\delta(x_0)$ and $\lambda$ is sufficiently large, approximately we have 
$$
\frac{z_0}{\lambda }\sim \mathcal{N}(x_0,\sigma^2),~ \sigma := \sqrt{\frac{x_0}{\lambda }}\, .
$$
When $\lambda$ approaches infinity,
according to the strong law of large numbers, we have %
$$
\lim_{\lambda\rightarrow\infty} \frac{z_0}{\lambda } = x_0
$$
with probability one. 
Therefore, as $\lambda$ increases, $\frac{z_0}{\lambda}$ tends to provide a more and more accurate approximation to $x_0$. Next, we will show that the same is true for general data distributions under certain regularity conditions.

\begin{theorem}
    Suppose $z_0\sim \emph{\mbox{Pois}}(\lambda x_0)$, $x_0\sim \mathcal{P}_0$, and the moment generating function $\bbE_{x_0}[e^{tx_0}]$ exists for $|t|<h$, where $h>0$, then random variable $$\hat{x}_0=z_0/\lambda$$ converges in distribution to $\mathcal{P}_0$ as $\lambda$ goes to infinity.
\end{theorem}

\begin{proof}
    The moment-generating function of $\hat{x}_0=z_0/\lambda$ can be expressed as
    \ba{
   \E[e^{t \frac{z_0}{\lambda}}] &= \E_{x_0}\E_{z_0\sim \text{Pois}(
\lambda x_0)}   [e^{t \frac{z_0}{\lambda}}]\notag\\
&=\E_{x_0}[\exp( \lambda x_0 (e^{\frac{t}{\lambda}}-1))].
    }
    Since $\lim_{\lambda\rightarrow\infty}  \lambda ( e^{\frac{t}{\lambda}}-1) = t ,$ we have 
    $$
    \lim_{\lambda\rightarrow\infty} \E[e^{t \frac{z_0}{\lambda}}]  = \E_{x_0}[e^{t x_0}],
    $$
    where the right-hand side is the moment generating function of $x_0\sim \mathcal{P}_0$. 
    
\end{proof}

Therefore, we will focus on modeling the distribution of latent count $z_0\in \mathbb{N}_0$, from which we can recover $x_0$ as $$\hat x_0 = g\left(\frac{z_0}{\lambda }\right),$$ where $g(y)$ is a function that maps its input $y$ into the domain of $x_0$ and the scaling parameter $\lambda$  controls the mapping accuracy.  For example, if $x_0\in\bbR_{\ge0}$, then $g(y)=y$; if $x_0\in\mathbb{N}_0$ is a count variable, then $g(y)=\text{round}(y)$, which rounds its input $y$ to the nearest integer; if $x_0\in\{0,1\}$ is a binary variable, then $g(y) = \mathbf{1}(y> 0)$; if $x_0$ is a categorical variable represented as a one-hot vector, then $g(y)$ outputs a one-hot vector whose non-zero location is at the dimension that $y$ takes its maximum value; and if $x_0$ is a simplex-constrained probability vector, then $g(y) = y/\|y\|_1$. 

In practical scenarios, the scaling parameter $\lambda$ does not necessarily need to be extremely large in order to achieve a good approximation. For instance, considering a distribution with a mean of 0.5, using a scaling parameter of $10$ or $100$ would result in a signal-to-noise ratio of $5$ or $50$, as measured by the following expression: $$\mathbb{E}_{x_0}\left[\frac{\left(\mathbb{E}[z_0\mid x_0]\right)^2}{\text{Var}(z_0\mid x_0)}\right]=\lambda\mathbb{E}[x_0].$$

\subsection{Thinning Can be Reversed with Shifted Poisson}
A well-known property of the Poisson distribution is that using  a binomial distribution, one can thin a Poisson random variable  into another Poisson random variable with a lower rate ($e.g.$, see page 163 of \citet{casella2021statistical}). Specifically, thinning $$p(x) = \mbox{Pois}(x;\lambda)$$ via $$p(y\given x)=\mbox{Binomial}(y;x,\pi),$$ where $\pi\in[0,1]$, leads to $$p(y) = \E_{x\sim p(x)}[p(y\given x)] = \mbox{Pois}(y;\pi\lambda).$$ Denote $x\sim \mbox{Pois}_{m}(\lambda)$ as a shifted Poisson distribution with probability mass function $$\mbox{Pr}(x=k) =\frac{\lambda^{k-m}e^{-\lambda}}{(k-m)!}, ~~k\in\{m,m+1,\ldots\}.$$ To reverse from $p(y)$ to $p(x)$, we show what is needed is a shifted Poisson distribution as
$$
p(x\given y)=\mbox{Pois}_{ y}(x;(1-\pi) \lambda),
$$
which is the same as adding a random discrete jump of $\mbox{Pois}((1-\pi) \lambda)$ into $y$ to generate $x$. In other words, we have the following Lemma:
\begin{lemma}\label{lem:1}
The Poisson-binomial bivariate count distribution and shifted-Poisson Poisson bivariate count distribution shown below are equivalent to each other:
\ba{
&p(x,y\given \lambda,\pi) = \emph{\mbox{Binomial}}(y;x,\pi) \emph{\mbox{Pois}}(x;\lambda)\notag\\
&~~~~~~~~~~=\emph{\mbox{Shifted-Pois}}_{y}(x;(1-\pi)\lambda) \emph{\mbox{Pois}}(y;\pi\lambda).\label{eq:b_count}
}\label{lemma}
\end{lemma}
The proof is straightforward by computing the probability mass functions of these two bivariate count distributions and showing they are equivalent to each other.  We also note that Lemma~\ref{lem:1} presented here can be considered as a specific instance of Lemma 4.1 in \citet{zhou2012beta}. The aforementioned lemma shows the equivalence between two scenarios: 1) drawing a total Poisson-distributed random count and allocating it among $K$ distinct categories via a multinomial distribution, and 2) independently drawing $K$ Poisson-distributed random counts. In the context of learning to jump, Lemma~\ref{lem:1} serves as a restatement of that lemma, tailored to the framework and constraints of this specific problem, with the particular configuration of $K=2$ representing the number of categories at each jump step.

\subsection{Count-thinning-based Encoder}
Starting by drawing a latent count
$z_0\sim q(z_0\given x_0) =  \mbox{Pois}(z_0;\lambda x_0)
$ and then gradually thinning it towards zero, we define a forward thinning process of $T$ time steps as
\ba{
q(z_{1:T}\given z_0)  &=\textstyle \prod_{t=1}^T q(z_{t}\given z_{t-1})
\notag\\
&
= \textstyle\prod_{t=1}^T \mbox{Binomial}\left(z_t;z_{t-1},{\frac{\alpha_t}{\alpha_{t-1}}}\right),
}
where $\{\alpha_t\}_{0,T}$ is the set of  thinning coefficients satisfying $$1=\alpha_0>\alpha_1>\ldots > \alpha_T \rightarrow 0.$$ By construction,  we have $$q(z_t\given z_0)=\mbox{Binomial}\left(z_t;x_{0},\alpha_t\right).$$ As $\alpha_T \rightarrow 0$, we have $$q(z_T\given z_0)=\mbox{Binomial}\left(z_T;x_{0},\alpha_T\right) \rightarrow  \delta_0(z_T),$$ which becomes a point mass at $z_T=0$. In other words, the Poisson thinning process possesses an absorbing state precisely at zero.

A key property of the forward thinning process, as suggested by Lemma~\ref{lemma}, is that the marginal distribution of the latent count at any $t$  remains to follow a Poisson distribution as
\ba{
q(z_t\given x_0) &= \E_{z_{0}\sim \text{Pois}(\lambda x_0)}[ q(z_t\given z_0)]\notag\\ &= \mbox{Pois}(z_t;\lambda\alpha_t x_0).\label{eq:marginal}
}
With \eqref{eq:b_count} and \eqref{eq:marginal}, we can show another key property of the thinning process:
The thinning from $z_{t-1}$ to $z_t$ can be reversed by a conditional posterior following the shifted-Poisson distribution  as 
\ba{
&q(z_{t-1} \given z_t,x_0)\notag\\
& = \frac{\mbox{Binomial}\left(z_t;z_{t-1},{\frac{\alpha_t}{\alpha_{t-1}}}\right) \mbox{Pois}(z_{t-1};\lambda\alpha_{t-1} x_0)}{ \mbox{Pois}(z_{t};\lambda\alpha_t x_0)}\notag\\
&=\mbox{Pois}_{ z_t}(  z_{t-1};\lambda (\alpha_{t-1}-\alpha_t) x_0),
}
which is key 
to deriving a tractable variational lower bound. %

\subsection{Count-thickening-based Decoder}

Let us denote $$f_{\theta}(z_{t},t)\ge 0$$ as a non-negative nonlinear function whose input consists of the count at time step $t$ and the time embedding for $t$. We start the count-thickening process at $z_t=0$, and at time step $t-1$, we add a jump of $$\mbox{Pois}(\lambda (\alpha_{t-1}-\alpha_t)f_{\theta}(z_{t},t))$$ into the previous accumulative count $z_{t}$ to arrive at the updated accumulative count $z_{t-1}$ at time $t-1$. More specifically, starting at $z_T\sim \mbox{Pois}(0)$, which means $z_T=0$ almost surely, 
we define a decoder with $T$ count-thickening steps, each of which corresponds to a shifted Poisson distribution, expressed as
\ba{
&p(z_{0:T-1}\given z_T=0) = \prod_{t=1}^T p_{\theta}(z_{t-1}\given z_t) %
\notag\\
&~~~~~ = \prod_{t=1}^T q(z_{t-1}\given z_t,\hat x_0 = f_{\theta}(z_{t},t)) %
\notag\\
&~~~~~ = %
\prod_{t=1}^T \mbox{Pois}_{z_{t}}(z_{t-1};\lambda(\alpha_{t-1}-\alpha_t) f_{\theta}(z_{t},t)).
\label{eq:decoder}
}

\subsection{Auto-encoding Variational Inference}
Below we provide the key steps of variational inference.
While it is intractable to compute the marginal distribution $p(x_0) = \E_{p_{\theta}(z_{0:T})}[p(x_0\given z_{0:T})]$, similar to the optimization in DDPMs \citep{sohl2015deep,ddpm}, we can minimize a negative evidence lower bound (ELBO) as %
\ba{
L &= -\E_{x_0\sim \mathcal P_0}\E_{q(z_{0:T}\given x_0)}\left[\ln \frac{p_{\theta}(z_{0:T}, x_0)}{q(z_{0:T}\given x_0)}\right] \notag\\
&=\E_{x_0\sim \mathcal P_0} \left[L_{-1} + \sum_{t=1}^{T}L_{t-1} + L_T\right],\label{eq:neg-elbo-loss}
}
where
\ba{
 L_{t-1} &=\E_{q(z_{t}\given x_0)}[\mbox{KL}(q(z_{t-1}\given z_t,x_0)\divbar p_{\theta}(z_{t-1}\given z_t) )]
 \label{eq:L_t_1}
 }
 for $1\le t\le T$ and
\ba{
 L_{-1}&=\bbE_{q(z_0\given x_0)}[-\ln p_{\theta}(x_0\given z_0)]
\\
 L_T&=\E_{q(z_0\given x_0)}[\mbox{KL}(q(z_T\given z_0)\divbar p(z_T))].
}
As $L_T\approx 0$, it can be ignored. In practice, we do not consider $L_{-1}$ since $\bbE_{x_0}[L_{-1}]$ minimizes when $p_{\theta}(x_0\given z_0)=q(x_0\given z_0)$, which can be well approximated by a deterministic mapping $g(\frac{z_0}{\lambda})$ when the scaling parameter $\lambda$ is sufficiently large.

The Kullback--Leibler (KL) divergence term in \eqref{eq:L_t_1} has an analytic expression as
\ba{
&\mbox{KL}\lp q(z_{t-1}\given z_t,x_0)\divbar p(z_{t-1}\given z_t)\rp\notag\\
&=\lambda
 (\alpha_{t-1}-\alpha_t)\notag\\
 &~~~~~~~~~~\times [ x_0(\ln x_0-\ln f_{\theta}(z_t,t)) - (x_0-f_{\theta}(z_t,t))]\notag\\
& = \lambda
 (\alpha_{t-1}-\alpha_t) D_{\varphi}(x_0, f_\theta(z_t,t)),\notag
}
where $D_{\varphi}(p, q)$ is the relative entropy, which is a Bregman divergence \citep{banerjee2005clustering} induced by the differentiable, strictly convex function $\varphi(x) = x\ln(x)$ such that
\ba{
D_{\varphi}(p,q) &= \varphi(p)-\varphi(q)-(p-q)^T\nabla_q \varphi(q) \notag\\
&= p( \ln{p}-\ln{q})-(p-q).\label{eq:RE}
}
We divide $ L_{t-1}$ by ${\lambda(\alpha_{t-1}-\alpha_t)}$ to define a re-weighted negative ELBO as

\ba{
\tilde L&=%
 \E_{x_0\sim \mathcal P_0} \left[\sum_{t=1}^T \tilde L (x_0,t)\right]
\notag\\
 \tilde L (x_0,t)& = \E_{z_0\sim q(z_0\given x_0)} \E_{q(z_{t}\given z_0)}[D_{\varphi}(x_0, f_\theta(z_t,t))].
\label{eq:loss}
}

As $D_{\varphi}(p,q)\ge 0$, we have $\tilde L\ge 0$ and hence we can monitor how close $\tilde L$ is to zero to assess convergence. Since $\E_{x_0\sim \mathcal P_0}[\tilde L (x_0,t) ]$ can also be written as $$ \E_{z_t\sim q(z_{t})}\E_{x_0\sim q(x_0\given z_t)}[D_{\varphi}(x_0, f_\theta(z_t,t))],$$ using the property of the Bregman divergence \citep{banerjee2005clustering}, we also know that the loss is minimized at $\theta^*$ when $f_{\theta^*}(z_t, t)=\bbE[x_0\given z_t, t]$ for all $z_t\sim q(z_t)$.

Our algorithm is straightforward to implement:

\noindent 1) Sample $t\in \{1,\ldots,T\}$ uniformly at random  and draw 
\ba{
z_t\sim \mbox{Pois}(\lambda \alpha_t x_0),~x_0\sim \mathcal P_0(x_0); %
\notag
}
2) Optimize $\theta$ with gradient 
\ba{{\small\nabla_{\theta} D_{\varphi}(x_0, f_\theta(z_t,t)) = \nabla_{\theta}[ f_{\theta}(z_t,t)-x_0 \ln f_{\theta}(z_t,t) ].}
\notag
}

Note that \eqref{eq:loss} can be readily extended to continuous time as
$$
\int_0^1\E_{x_0\sim \mathcal P_0}\E_{q(z_{t}\given x_0)}[D_{\varphi}(x_0, f_\theta(z_t,t))]dt,
$$
where $q(z_t\given x_0)=\mbox{Pois}(z_t; \lambda\alpha_tx_0)$ and $\alpha_t=\alpha(t)$ is a monotonically-decreasing function defined on $[0, 1]$ such that $\alpha(0)=1$, $\alpha(1)\approx 0$, and $1>\alpha_s>\alpha_t>0$ for $0<s<t<1$.

For data generation, we let $z_T=0$, perform ancestral sampling via \eqref{eq:decoder} to draw $z_0$, and then let either $\hat{x}_0 =g(z_0/\lambda)$ or $\hat{x}_0 =f_{\theta}(z_1,1)$ to produce a random generation. We summarize the training and sampling algorithms in Algorithms~\ref{algo:train} and \ref{algo:sample}, respectively.

\begin{algorithm}[t]
\caption{Training\label{algo:train}}
\begin{algorithmic}[1]
    \REQUIRE Dataset $\cD$, Mini-batch size $B$, Scaling parameter $\lambda$, timesteps $T$, thinning coefficients $\{\alpha_t\}_{t=1}^T$, and decoder deep network $f_\theta$
    \REPEAT
    \STATE Draw a mini-batch $X_0=\{x_0^{(i)}\}_{i=1}^B$ from $\cD$
    \FOR{$i=1\ \text{to}\ B$}
        \STATE $t_i\sim\mbox{Uniform}(\{1,\ldots,T\})$
        \STATE $z_t^{(i)}\sim\mbox{Poisson}(\lambda\alpha_{t_i}x_0^{(i)})$

    \STATE Compute loss $\mathcal L_i=D_\varphi(x_0^{(i)}, f(z_t^{(i)}, t_i))$
    \ENDFOR
    \STATE Perform SGD with $\frac{1}{B}\nabla_{\theta}\sum_{i=1}^B \mathcal L_i
 $
    \UNTIL{converge}
\end{algorithmic}
\end{algorithm}

\begin{algorithm}[t]
\caption{Sampling\label{algo:sample}}
\begin{algorithmic}[1]
    \REQUIRE Scaling parameter $\lambda$, timesteps $T$, thinning coefficients $\{\alpha_t\}_{t=1}^T$, decoder deep network $f_\theta$ and constraint function $g$
    \STATE Initialize a mini-batch $z_T=\mathbf{0}$
    \STATE $z_t\gets z_T$
    \FOR{$t=T\ \text{to}\ 1$}
        \STATE $\hat{x}_0 \gets f_\theta(z_t, t)$
        \STATE $z_{t-1}\sim z_t + \mbox{Poisson}(\lambda(\alpha_{t-1}-\alpha_t)\hat{x}_0)$
    \ENDFOR
    \STATE $x_0\gets g(z_0/\lambda)$, or $x_0=g(\hat x_0)$
    \STATE \textbf{return} $x_0$
\end{algorithmic}
\end{algorithm}

\section{Related Work} %
The proposed learning to jump provides a new framework to construct DGMs. Below we discuss several representative DGMs and how JUMP models differ from them.

\textbf{VAEs and GANs.} ~
Both VAEs and GANs utilize deep neural networks in their data-generating process. They typically forward propagate a random noise once through an encoder, parameterized by a deep neural network, to generate a random data sample. This is different from learning to jump, which needs to iterate its generation through the same deep neural network multiple times before producing a single realistic sample in the original data space.    VAEs \citep{kingma2013auto,rezende2014stochastic}  are learned by maximizing the lower bound of an intractable log marginal likelihood,
whereas GANs \citep{goodfellow2014generative} are learned under a min-max adversarial game between a discriminator and a generator. When applied to image generation, VAEs are known to generate blurry images while GANs are known to be susceptible to training instability and dropping data density modes.

A wide variety of techniques have been developed over the years to improve their performance.  For VAEs, a considerable amount of effort has been spent on  %
constructing more expressive tractable variational posteriors \citep{ranganath2016hierarchical,huszar2017variational,yin2018semi,zhang2018advances,molchanov2019doubly,titsias2019unbiased} and %
improving the decoder architecture to generate more photo-realistic images \citep{razavi2019generating,maaloe2019biva,vahdat2020nvae}. Whereas for GANs, steady progress has been made on improving training stability \citep{arjovsky2017wasserstein,gulrajani2017improved,miyato2018spectral,mescheder2018training}, generation fidelity \citep{radford2015unsupervised,brock2018large,karras2019style,sauer2022styleganxl}, and mode coverage and data efficiency \citep{zhao2020differentiable,karras2020training,yang2021data,wang2022diffusion}.

\textbf{Learning to denoise.}
~Both score-based generative models \citep{scorematching} and DDPMs \citep{ddpm} can  be considered as representative DGMs developed under the learning-to-denoise framework, which generates a random sample by iteratively refining its generation through a deep neural network. From the Bayesian perspective, the method of learning to denoise can be implemented under an auto-encoding variational inference framework. Defining data generation via a multi-stochastic-layer generative network %
and introducing a fixed hierarchical variational encoder for inference, 
one may construct an ELBO of the log marginal likelihood to derive both the training and inference algorithms \citep{sohl2015deep,ddpm}.

Specifically, as in DDPMs, one may take a Markov diffusion chain as the encoder, which gradually corrupts the data towards pure Gaussian noise by repeatedly mixing it with Gaussian noise at various scales. This provides the training data to supervise the learning of a reverse Markov diffusion chain. After being trained, this reverse chain iterates through the same deep neural network, whose inputs include the time embedding of the current diffusion step, to gradually refine a Gaussian noise into a noise-free data generation. 

We note that DDPMs can also be formulated as either performing denoising score matching in a discrete-time setting or solving stochastic differential equations in a continuous-time setting  \citep{song2021scorebased}.
While the development of the JUMP models under the learning-to-jump framework mimics that of DDPMs under the learning-to-denoise framework, the JUMP models can be viewed as neither score matching nor stochastic differential equations. This is because the generations of JUMP models take count values   that are not continuous. 

Learning to denoise has also been generalized to model categorical data, where the noise corruption corresponds to randomly transiting a categorical observation to some other category under a pre-defined transition probability matrix \citep{hoogeboom2021argmax,austin2021structured}. Inspired by the success of masked language models in natural language processing \citep{devlin2018bert}, one may further augment the existing categories with a mask category whose self-transition probability is one \citep{austin2021structured}. In other words, the mask category is an absorbing category. Consequently, unmasking becomes an essential part of the denoising process for data generation. When combined with an appropriate pretrained encoder-decoder with a tokenized discrete latent space, such as that provided by VQ-VAE \citep{van2017neural} or VQ-GAN \citep{esser2021taming}, learning to unmask and/or denoise in this discrete space has led to strong performance in both unconditional and text-conditioned image generation \citep{gu2022vector,hu2022global,chang2022maskgit,chang2023muse}. Related to learning to unmask which has the mask category as its unique absorbing state, the proposed learning to jump also has a unique absorbing state, which is 0. A distinction is that the mask category is a nonexistent fictitious category, while 0 is a possible true data value in learning to jump.

\textbf{Learning to reverse other types of corruptions.}
~Several recent works have all tried to come up with a method that mimics learning to denoise but changes the data corruption from adding Gaussian noise to another type of corruption. Representative examples include learning to de-blur \citep{hoogeboom2023blurring}, learning to reverse heat dissipation \citep{rissanen2023generative}, and learning to reverse arbitrary and even noiseless/cold image transforms \citep{bansal2022cold, daras2022soft}. These variations of learning to denoise ultimately all boil down to minimizing an $L_2$ loss between the clean data and the predicted reconstruction of the corrupted version of the data. From this perspective, learning to jump differs from all of them in having a relative entropy-based loss, as shown in \eqref{eq:loss}, that is different from an $L_2$ loss. Another notable difference is that the starting point of the reverse chain in learning to jump is a point mass at 0, rather than some random noise from a fixed prior distribution.

\section{Experiments}

To demonstrate the power and versatility of the proposed learning-to-jump framework for generative modeling, we evaluate JUMP models on a diverse set of non-negative data, ranging from univariate non-negative data of various types, document term-frequency count vectors and TF-IDF vectors obtained from two representative text corpora, to natural images whose pixel values lie between 0 and 255. For data such as natural images that typically exhibit no strong sparsity, skewness, heavy-tailedness, or overdispersion, learning-to-denoise-based DGMs, such as DDPMs, have already been proven to perform well. In this case, we don't expect JUMP models to provide unique advantages in faithfully regenerating the original data distribution. However, when the data is highly sparse, skewed, heavy-tailed, and/or overdispersed, we expect the proposed JUMP models to behave differently, potentially providing distinct advantages in data regeneration, as confirmed by a rich set of experiments shown below. Our code is available at \url{https://github.com/tqch/poisson-jump}.

\subsection{Univariate Non-negative Data}
\label{sec:synthetic-1}

\begin{figure*}[th]
    \centering
    \includegraphics[%
    width=.9\textwidth]{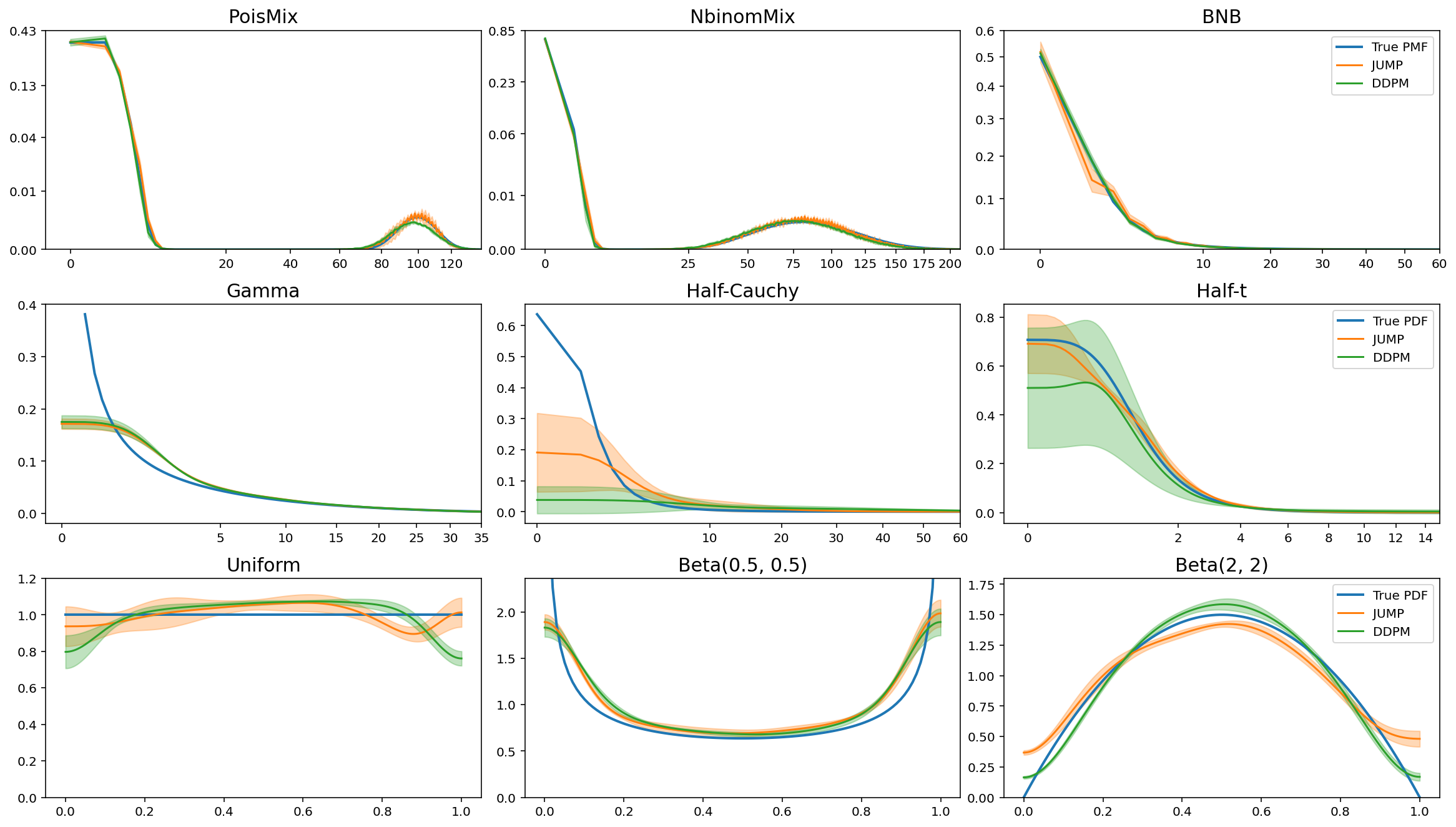}
    \caption{Visual comparison of the true and regenerated probability distributions by both DDPM and JUMP, across 9 univariate synthetic datasets described in Section~\ref{sec:synthetic-1}.}%
    \label{fig:univariate_synthetic}
\end{figure*}

We consider three types of univariate non-negative data $x$, including $x\in\mathbb{N}_0$, $x\in[0,\infty) $, and $x\in [0,1]$.
For $x\in\mathbb{N}_0$, we synthesize three count datasets, which are characteristic of sparsity, skewness, and overdispersion, from a bi-modal Poisson mixture~(PoisMix)
$$
0.9 \cdot \mbox{Pois}(1) + 0.1 \cdot \mbox{Pois}(100),
$$
a bi-modal negative binomial (Nbinom) mixture~(NbinomMix)
$$
0.75 \cdot \mbox{Nbinom}(1, 0.9) + 0.25 \cdot \mbox{Nbinom}(10, 0.1),
$$
and a beta-negative-binomial~(BNB) distribution with probability mass
$$
P(k)=\int_0^1\mbox{Nbinom}(k; 1, p)\mbox{Beta}(p; 1.5, 1.5)d p,~k\in \mathbb{N}_0 \, . %
$$
We also experiment on continuous non-negative data from either skewed or heavy-tailed distributions, including $\mbox{Gamma}(0.5, 0.05)$, a half-Cauchy with probability density
$$
p(x)\propto (1+x^2)^{-1},~x\ge 0,
$$ and a half-t distribution with probability density
$$
p(x)\propto\lp1+\frac{x^2}{2}\rp^{-\frac{3}{2}},~x\ge0.
$$
We further consider $x\in [0,1]$ drawn from $\mbox{Uniform}(0,1)$, $\mbox{Beta}(0.5, 0.5)$, or $\mbox{Beta}(2, 2)$, which correspond to flat, convex, and concave shaped probability density functions, respectively, that are all symmetric around the mean $x=0.5$.

We draw 100,000 random samples from each distribution to form the  training data. All the evaluated models in the experiment use the same 3-layer MLP architecture with 128 hidden units and leaky-ReLU activation. For D3PMs \cite{austin2021structured}, we truncate the distributions at their $0.999$ quantiles to ensure the number of states is finite. For the proposed JUMP, we set $\lambda$ to 10 for all the univariate datasets except for $\mbox{Uniform}$, $\mbox{Beta}(0.5, 0.5)$ and $\mbox{Beta}(2, 2)$, where we use $\lambda=100$.

We report the mean and standard deviation of the Wasserstein-1 distances between the
true/empirical distribution of the training data and the empirical distribution of 100,000 generated samples over 5 independent runs. 
We compute Wasserstein-1 \citep{COTFNT} as
$$
W_1(p, q)\triangleq \int_{\bbR\times\bbR} |x-y|d\pi(x, y) = \int_{\bbR}|P-Q|(x)dx,
$$
where $p, q$ are arbitrary univariate distributions, $\pi(x, y)$ is their coupling, and $P, Q$ are cumulative distribution functions of $p, q$ respectively. When $p, q$ are two empirical distributions with the same size $n$, it reduces to compute $\|\text{sort}(X)-\text{sort}(Y)\|_1/n$, where $X,Y$ are $n$-dimensional data vectors of $p, q$.

\begin{table}[t]
\caption{
Wasserstein-1 distances between the empirical discrete distribution of the generated random samples and \textbf{Top:} the true training distribution, \textbf{Middle \& Bottom:} the empirical distribution of the training set. The unit in the Bottom table is $10^{-2}$.}
\label{tab:univariate_synthetic}

\begin{adjustbox}{width=\columnwidth,center}
\begin{tabular}{c|ccc}
    \toprule
    &\multicolumn{3}{c}{Discrete $x\in \mathbb{N}_0$}\\
    \cmidrule(lr){2-4}
    & PoisMix & NbinomMix &  BNB \\
    \midrule
    DDPM~\cite{ddpm} & $1.48\pm 0.37$ & $2.17\pm 0.77$ & $4.49\pm 4.41$\\
    D3PM Unif~\cite{austin2021structured} & $23.53\pm 0.48$ & $22.31\pm 0.09$ & $2.53\pm 0.02$\\
    D3PM Gauss~\cite{austin2021structured} & $17.29\pm 1.41$ & $19.57\pm 0.67$ & $2.59\pm 0.02$\\
    JUMP (ours) & $\mathbf{0.85\pm 0.41}$ & $\mathbf{1.84\pm 0.44}$ & $\mathbf{1.11\pm 0.35}$\\
    \bottomrule\addlinespace
    \toprule
    &\multicolumn{3}{c}{Continuous $x\in [0,\infty)$}\\
    \cmidrule(lr){2-4}
    & Gamma & Half-Cauchy & Half-t \\
    \midrule
    DDPM~\cite{ddpm} & $0.85\pm 0.26$ & $17.05\pm 18.58$ & $ 0.51\pm 0.76$\\
    JUMP (ours) & $\mathbf{0.60\pm 0.27}$ & $\mathbf{3.56\pm 0.76}$ & $\mathbf{0.11\pm 0.02}$\\
    \bottomrule\addlinespace
    \toprule
    &\multicolumn{3}{c}{Continuous $x\in[0,1]$}\\
    \cmidrule(lr){2-4}
    & Uniform & $\mbox{Beta}(0.5, 0.5)$ & $\mbox{Beta}(2, 2)$ \\
    \midrule
    DDPM~\cite{ddpm} & $1.40\pm 0.26$ & $\mathbf{1.24 \pm 0.44}$ & $1.49\pm 0.35$\\
    JUMP (ours) & $\mathbf{1.39\pm 0.42}$ & $1.57\pm 0.17$ & $\mathbf{1.30\pm 0.34}$\\
    \bottomrule
\end{tabular}
\end{adjustbox}\vspace{5mm}
\end{table}

We summarize the results in Table\,\ref{tab:univariate_synthetic} and Figure\,\ref{fig:univariate_synthetic}. 
The results on $\mbox{Gamma}(0.5,0.05)$ and the 3 datasets supported on $[0,1]$ all  suggest that when the data is well behaved in the sense that there is no strong sparsity, skewness, heavy-tailedness, or overdispersion, the proposed JUMP has %
comparable %
performance to DDPM in terms of the Wasserstein-1 metric.
However, the results on the other 5 training datasets, which all clearly exhibit sparsity, skewness, heavy-tailedness, and/or overdispersion, show that the proposed JUMP convincingly outperforms DDPM (and also D3PM on discrete data). %

\subsection{Sparse and Heterogeneous Data}
Bag-of-words~(BOW) and term frequency-inverse document frequency~(TF-IDF) are two common types of document representations. Both are known to be highly sparse and heterogeneous. We prepare two datasets for the document generation task: 20 Newsgroups\footnote{\url{http://qwone.com/~jason/20Newsgroups/}} and NeurIPS\footnote{\url{https://www.kaggle.com/datasets/benhamner/nips-papers}}.

The 20 Newsgroups dataset comprises 18,846 news posts on 20 topics. The NeurIPS  dataset is a collection of 7241 papers published in NeurIPS from 1987 to 2016. After standard preprocessing ($e.g.$, removing stop-words), we represent each document in 20 Newsgroup and NeurIPS  with a count / TF-IDF vector of 8934 and 12038 dimensions, respectively. We set $\lambda$ as 10 for both datasets. For the evaluation of BOW, we consider two summary statistics of the documents, including the sparsity and the length between the true data and generated samples, while for TF-IDF we consider the sparsity and the $\ell_1$-norm. 

We report the Wasserstein-1 distances of the summary statistic distributions in Table~\ref{tab:doc_wdist}. The results show that JUMP has done extremely well in recovering the inherent sparsity of document-type data. In contrast, DDPM has almost completely failed in capturing the sparsity patterns of the TF-IDF documents. Overall, JUMP consistently outperforms DDPM in all scenarios and metrics except for the $\ell_1$-norm of the NeurIPS dataset using TF-IDF representation, where the performance is comparable to that of DDPM.

\begin{table}[t]
    \caption{\textbf{Top:} Wasserstein-1 distances of summary statistics~(sparsity and length) between true samples and generated samples in BOW representations. \textbf{Bottom:} Wasserstein-1 distances of summary statistics~(sparsity and $\ell_1$-norm) between true data and generated samples in TF-IDF representations.\label{tab:doc_wdist}}
    \begin{adjustbox}{width=\columnwidth,center}
    \begin{tabular}{c|cccc}
    \toprule
    &\multicolumn{4}{c}{BOW}\\
    \cmidrule(lr){2-5}
    &\multicolumn{2}{c}{20 Newsgroup} & \multicolumn{2}{c}{NeurIPS}\\
    \cmidrule(lr){2-3}\cmidrule(lr){4-5}
    & Sparsity~(\textperthousand) & Length & Sparsity~(\%) & Length\\
    \midrule
    DDPM~\cite{ddpm} & $5.49\pm 0.04$ & $81.80\pm 10.2$ & $3.46\pm 0.38$ & $191.86\pm 98.4$\\
    JUMP (ours) & $\mathbf{2.65\pm 0.46}$& $\mathbf{35.43\pm 4.16}$& $\mathbf{2.40\pm 0.29}$& $\mathbf{95.56\pm 29.36}$\\
    \bottomrule\addlinespace
    \toprule
    &\multicolumn{4}{c}{TF-IDF}\\
    \cmidrule(lr){2-5}
    &\multicolumn{2}{c}{20 Newsgroup} & \multicolumn{2}{c}{NeurIPS}\\
    \cmidrule(lr){2-3}\cmidrule(lr){4-5}
    & Sparsity~(\textperthousand) & $\ell_1$-norm & Sparsity~(\%) & $\ell_1$-norm\\
    \midrule
    DDPM~\cite{ddpm} & $693.21\pm 5.45$ & $0.65\pm 0.02$ & $66.59\pm 1.74$ & $\mathbf{0.20 \pm 0.02}$\\
    JUMP (ours) & $\mathbf{3.18\pm 0.06}$& $\mathbf{0.41\pm 0.02}$& $\mathbf{6.23\pm 0.00}$& $0.26\pm 0.00$\\
    \bottomrule
    \end{tabular}
    \end{adjustbox} %
\end{table}

\subsection{Natural Images}

Images are by nature high-dimensional data with 256 ordinal pixel values. Typically, the pixel values in images are neither sparse nor heavy-tailed, and hence as analyzed in Section~\ref{sec:synthetic-1}, we do not expect JUMP to outperform DDPM. Our experiments show that while JUMP currently underperforms DDPM in modeling natural images, using the same UNet architecture and diffusion schedule that have been well-tuned to suit DDPM, it can nevertheless generate realistic-looking natural images and achieve comparable evaluation results on standard metrics, including Fréchet Inception distance~(FID)~\cite{heusel2017gans} and Inception score~(IS)~\cite{salimans2016improved}, to diffusion models that operate on non-integer latent states. We present uncurated randomly-generated images in Figure~\ref{fig:cifar10} and report the FID and IS metrics in Table~\ref{tab:fid_is}.

\section{Limitations and Future Work}
A notable limitation of the learning-to-denoise framework is that it often needs to iterate through the same denoising deep neural network hundreds or even thousands of times to refine its generation, which increases the computational cost of data generation by orders of magnitude compared to VAEs and GANs of similar sizes.
The proposed learning-to-jump framework has the same limitation, as it also requires iterating $T$ times through $f_{\theta}(z_t,t)$ to generate a single output, where $T$ often needs to be sufficiently large, $e.g$, $T=1000$, especially for high-dimensional data %
whose different dimensions exhibit complex dependencies.
For learning to denoise, a variety of methods have been proposed to accelerate the generation, but often at the expense of somewhat compromised generation quality when $T$ is limited to a small number \citep{song2021denoising,luhman2021knowledge,kong2021fast,xiao2022tackling,salimans2022progressive,zheng2023truncated,lu2022dpmsolver}. How these acceleration techniques developed for learning to denoise can be extended for the proposed learning-to-jump framework is a research topic worth further investigation.  

Another limitation, as shown by the results %
in Figure~\ref{fig:univariate_synthetic} and Table~\ref{tab:univariate_synthetic} and the image generation results in Table~\ref{tab:fid_is}, is that for ``normal'' data that are not sparse, skewed, heavy-tailed, or heterogeneous, JUMP often trails behind DDPM, under the same UNet architecture and diffusion schedule that have been tailored for DDPM. How to further improve JUMP for ``normal'' data, such as by developing a customized model architecture and optimizing the diffusion schedule, is worth further investigation. 

\begin{figure}[t]
    \centering
    \includegraphics[width=.95\columnwidth]{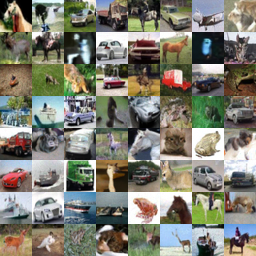}
   \vspace{-3mm}
    \caption{Uncurated randomly-generated samples by JUMP trained  on CIFAR-10. \label{fig:cifar10}}
\end{figure}

\begin{table}[t]
\caption{Comparison of different models on {CIFAR-10}.\label{tab:fid_is}}
\begin{adjustbox}{width=\columnwidth,center}
\begin{tabular}{clcc}
    \toprule
    Latent space & Model & FID~($\downarrow$) & IS~($\uparrow$) \\
    \midrule
    \multirow{2}{*}{Real} & DDPM~\cite{ddpm} & 3.17 & 9.46 \\
    & Bit Diffusion~\cite{chen2023analog} & 3.48 & -\\
    \midrule
    \multirow{2}{*}{Categorical} & D3PM Gauss+Logistic~\cite{austin2021structured} & 7.34 & 8.56\\
    & $\tau$LDR-10~\cite{campbell2022a} & 3.74 & 9.49 \\
    \midrule
    Integer & JUMP (ours) & 4.80 & 9.04 \\
    \bottomrule
\end{tabular}
\end{adjustbox}%
\end{table}

\section{Conclusion}
Iterative-refinement-based deep generative models (DGMs) are typically developed under a learning-to-denoise framework,  which includes diffusion and score-based generative models as representative examples. They have shown impressive performance in capturing the high-dimensional distribution of natural images, but may not perform that well when the data is characterized by sparsity, skewness, heavy-tailedness,  overdispersion, and/or heterogeneity. 
To this end, we propose learning to jump as a novel generative-modeling  framework, which is well suited to model sparse, skewed, heavy-tailed, overdispersed, and/or heterogeneous data, and hence generalizes the applicability of DGMs into broader settings. Experimental results on a diverse set of data of various types demonstrate the unique behaviors and modeling potentials of the learning-to-jump-based DGMs. In particular, for high-dimensional data, we recommend using learning-to-jump in lieu of learning-to-denoise when the training data are highly sparse and heterogeneous.

\section*{Acknowledgments}
The authors acknowledge the support of NSF-IIS 2212418, the Fall 2022 McCombs REG award, the NSF AI Institute for Foundations of Machine Learning (IFML), and Texas Advanced Computing Center (TACC).

\bibliographystyle{icml2023}
\bibliography{reference.bib,ref.bib}
\clearpage
\newpage
\onecolumn
\appendix
\begin{center}
    \Large{\textbf{\\Learning to Jump: Appendix}}
\end{center}

\section{Hyperparameter Settings}

\subsection{Diffusion Schedule} %
We have not conducted an extensive tuning of the diffusion schedule. As a default choice, we utilize the beta-linear schedule introduced in the work of \citet{ddpm}. We set $\beta_1$ to 0.001 by default. The value of $\beta_T$ is selected such that the log-SNR (Signal-to-Noise Ratio) will be approximately $-12$ on average at the end of the forward chain, ensuring that the loss of the last time step $L_T$ is approximately 0.

\subsection{Scaling Parameter} %
Intuitively speaking, the scaling parameter $\lambda$ controls how close the initial latent count distribution and original data distribution are. Specifically, the larger $\lambda$ is, the more precise the transform $f: z_0\mapsto z_0/\lambda$ will be to recover the $x_0$ in distribution. 
Based on our practical observations, we have found that starting with values of $10$ or $100$ for the noise schedule parameter is often suitable for most datasets, excluding image data. These values have shown promising performance as initial choices in our experiments. However, it is important to note that for image datasets, larger values often yield better results, and further experimentation and tuning are recommended to determine the optimal scaling parameter for a specific image dataset. Indeed, we find out that the Fréchet Inception Distance~(FID)~\cite{heusel2017gans} is highly sensitive to noise, even when the noise becomes imperceptible to humans. Figure~\ref{fig:fid-vs-lambda} illustrates the relationship between the scaling parameter $\lambda$ and the corresponding FID between the data distribution of CIFAR-10 and the reconstructed distribution obtained from the initial latent counts, $i.e.$, Poisson-randomized data.

\begin{figure}[h]
    \centering
    \includegraphics[width=.7\linewidth]{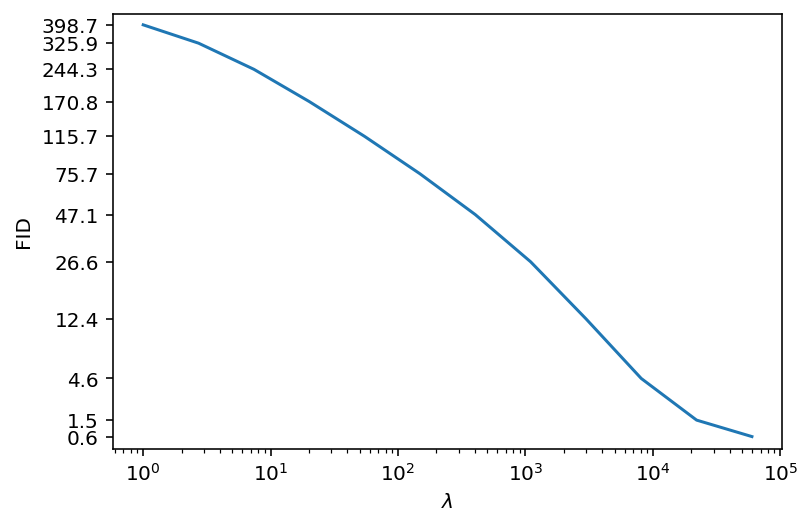}
    \vspace{-3mm}
    \caption{Relationship between FID of the initial latent count distribution and the scaling parameter $\lambda$ on CIFAR-10}
    \label{fig:fid-vs-lambda}
\end{figure}

\subsection{Training}

For all the univariate and document-type datasets, our models are trained for 600 epochs using the Adam optimizer~\cite{adam}. We use a fixed learning rate of 0.001 and the default values for the parameters $\beta_1=0.9$ and $\beta_2=0.999$.

In the case of CIFAR-10 image generation, we utilize the AdamW optimizer~\cite{adamw} with a learning rate of 0.0002 and a weight decay of 0.001. The JUMP model is trained for 3600 epochs. Unlike DDPM, we do not employ any learning rate warmup schedule in our training process.

\section{Model Architectures}

In the case of univariate and document datasets, we employ a Multi-Layer Perceptron (MLP) model architecture composed of three blocks. Each block consists of two fully-connected layers and one time embedding projection layer. Intermediate layers use Leaky-ReLU activation along with layer normalization \cite{ba2016layer}. For CIFAR-10 image generation, we use a UNet model architecture similar to the one used by \citet{nichol2021improved}. Our UNet model has three stages through downsampling and upsampling, which correspond to spatial dimensions of $32\times 32$, $16\times16$, and $8\times 8$. Each stage of the model consists of 3 residual blocks with 128 hidden channels followed by a self-attention layer except for the first stage. In addition, we use a dropout rate of $0.2$ for extra regularization.

\section{Binomial JUMP for Count Data}\label{sec:binomial}

We also consider a variant of the proposed JUMP model when the original data are already counts and hence Poisson randomization may not be necessary. Specifically, we remove the Poisson randomization step used by default in Poisson JUMP, resulting in a JUMP variant where the conditional posteriors of the reverse process follow the shifted-binomial distributions. We refer to this variant as binomial JUMP. The definition of the forward process of binomial JUMP can be expressed as
\ba{
q(z_{1:T}\given x_0)&=\prod_{t=1}^T q(z_{t}\given z_{t-1}),\\
q(z_{t}\given z_{t-1})&=\mbox{Binomial}\lp z_t;z_{t-1}, \frac{\alpha_t}{\alpha_{t-1}}\rp,
}
where $z_0=x_0$, $1=\alpha_0>\alpha_1>\ldots>\alpha_T$. It immediately follows that the marginal distribution of $z_t$ given $x_0$ is binomial: 
$$
q(z_t\given x_0)=\mbox{Binomial}(z_t;x_{0}, \alpha_t).
$$ 
If we know $z_t$ and $x_0$, then we have $z_t\le z_{t-1}\le x_0$ almost surely. 
Thus, with Bayes' rule, we have 
\ba{
q(z_{t-1}=m\given z_t,x_0) & \propto \mbox{Binomial}(m;x_0,\alpha_{t-1}) \mbox{Binomial}\left(z_t;m,\frac{\alpha_t}{\alpha_{t-1}}\right)\\
&\propto \frac{x_0!}{m! (x_0-m)!} \alpha_{t-1}^m (1-\alpha_{t-1})^{x_0-m} \frac{m!}{z_t! (m-z_t)!}\lp\frac{\alpha_t}{\alpha_{t-1}}\rp ^{z_t}\lp 1-\frac{\alpha_t}{\alpha_{t-1}}\rp^{m-z_t}\\
&\propto \frac{1}{(x_0-m)!(m-z_t)!}\left(\frac{\alpha_{t-1}-\alpha_t}{1-\alpha_{t-1}}\right)^m\\
&\propto \frac{(x_0-z_t)!}{(x_0-m)!(m-z_t)!}\left(\frac{\alpha_{t-1}-\alpha_t}{1-\alpha_{t}}\right)^{m-z_t}
\left(1-\frac{\alpha_{t-1}-\alpha_t}{1-\alpha_{t}}\right)^{x_0-m}
\\
&= \mbox{Binomial}\lp m-z_t; x_0-z_t, \frac{\alpha_{t-1}-\alpha_t}{1-\alpha_t}\rp, \label{eq:shifted-binomial}
}
where $z_t\le m \le x_0$.
Therefore, 
the conditional posterior is a shifted-Binomial distribution as  
\ba{
q(z_{t-1}\given z_t, x_0)=\mbox{Shifted-Binomial}_{z_t}\lp x_0-z_t, p_t\rp,~~p_t = \frac{\alpha_{t-1}-\alpha_t}{1-\alpha_t}.\label{eq:shifted-bino}
}
However, in this case, we cannot simply let $p_\theta(z_{t-1}\given z_t)=q(z_{t-1}\given z_t, \hat{x}_0 = f_{\theta}(z_t,t))$. This is because, in order for the KL divergence from the approximated conditional posterior  to the true one to be well defined, we will need to ensure that both $f_{\theta}(z_t,t)$ is a count and $f_{\theta}(z_t,t)\ge x_0$,  which are difficult to realize in practice using a non-linear function defined by the deep neural network-based $f_{\theta}$.
To address this issue, noticing that a binomial distribution $x\sim \mbox{Binomial}(n,p)$ can often be well approximated by a Poisson distribution $x\sim\mbox{Pois}(np)$ when $n$ is large and $np$ is small, we propose to approximate the shifted-binomial distribution in \eqref{eq:shifted-bino} with a shifted-Poisson distribution as 
$$
\hat{q}(z_{t-1}\given z_t,x_0)=\mbox{Shifted-Pois}_{z_t}\lp z_{t-1};p_t(x_0-z_t)\rp ,
$$
and define a Markovian reverse process as $p_\theta(z_{0:T-1}\given z_T=0)=\prod_{t=1}^Tp_\theta(z_{t-1}\given z_t)$, where
\ba{
p_{\theta}(z_{t-1}\given z_t) &= \sum_{n_t}\mbox{Shifted-Binomial}_{z_t}\lp z_{t-1}; n_t, p_t  %
\rp  \mbox{Pois}(n_t;\max(f_{\theta}(z_t,t)-z_t,0))\\ 
&= \mbox{Shifted-Pois}_{z_t}\lp z_{t-1}; p_t \max(f_{\theta}(z_t,t) - z_t,0)\rp.}
Similar to the derivation of Equation~\eqref{eq:neg-elbo-loss}, the approximated negative ELBO loss of binomial jump can be expressed as follows:

\ba{
L=-\bbE_{x_0}\bbE_{q(z_{1:T}\given x_0)}\left[\ln\frac{p_\theta(z_{1:T}, x_0)}{q(z_{1:T}\given x_0)}\right] = \bbE_{x_0}\left[L_0+\sum_{t=2}^TL_{t-1}+L_T\right]
}
where
\ba{
L_0&=\bbE_{q(z_1\mid x_0)}\left[-\ln p_\theta(x_0\mid z_1)\right]\\
L_{t-1}&=\bbE_{q(z_t\mid x_0)}\left[D_{\varphi}\lp p_t\lp x_0 - z_t\rp, p_t\max( f_\theta(z_t, t) - z_t,0)\rp\right], ~\text{ for } t=2,\ldots,T \label{eq:D_bino}\\
L_T&=\kl(q(z_T\mid x_0)\divbar p(z_T))
}
where $D_{\varphi}(\cdotv,\cdotv)$ denotes the relative entropy defined in \eqref{eq:RE}. More specifically, ignoring $p_t$ in \eqref{eq:D_bino}, we have
$$
D_{\varphi}\lp x_0 - z_t, \max( f_\theta(z_t, t) - z_t,0)\rp
=
(x_0 - z_t)\ln \frac{x_0 - z_t}{\max( f_\theta(z_t, t) - z_t,0)} - [(x_0-z_t)-\max( f_\theta(z_t, t) - z_t,0)].
$$
We observe that the relative entropy-based loss function of the binomial JUMP, as illustrated above, shares a close connection with the loss function of the Poisson JUMP, as shown in  \eqref{eq:loss}, with a clear distinction: In the binomial JUMP, we have $z_t\sim \mbox{Binomial}(x_0, \alpha_t)$, resulting in $z_t\leq x_0$, whereas in the Poisson JUMP, we have $z_t\sim \mbox{Pois}(\lambda \alpha_t x_0)$, allowing $z_t$ to potentially exceed $x_0$.

\end{document}